\documentclass{article}

\usepackage[preprint]{neurips_2026}
\usepackage{enumitem}
\usepackage{graphicx}
\usepackage{algorithm}
\usepackage{algorithmic}

\usepackage{etoolbox}
\usepackage{xspace}
\usepackage{amsthm}
\usepackage{multirow}
\usepackage{float}
\usepackage{amsmath}
\newtheorem{theorem}{Theorem}
\newtheorem{lemma}[theorem]{Lemma}

\usepackage[utf8]{inputenc} 
\usepackage[T1]{fontenc}    
\usepackage{hyperref}       
\usepackage{url}            
\usepackage{booktabs}       
\usepackage{amsfonts}       
\usepackage{nicefrac}       
\usepackage{microtype}      
\usepackage{xcolor}         

\title{Mesh-RL: Coupled subgrid reinforcement learning}

\author{%
  Behnam Gheshlaghi \\
  Independent Researcher
   \And
   Bahador Rashidi \\
   Independent Researcher
   \And
   Shahin Atakishiyev\thanks{Corresponding author. Email: shahin.atakishiyev@ualberta.ca} \\
   University of Alberta
  }

\begin{document}

\maketitle

\begin{abstract}
Reinforcement learning in large or sparse-reward environments suffers from slow temporal-difference reward propagation, as value information spreads only locally across the state space. We propose \emph{Mesh-RL}, a spatial domain-decomposition framework inspired by the finite element method and domain decomposition theory, which partitions the environment into overlapping subgrids and enforces \emph{boundary-consistent temporal-difference updates}. Such an approach enables localized learning while ensuring globally coherent value propagation. Unlike hierarchical or model-based approaches, Mesh-RL accelerates long-range credit assignment without modifying the reward function, Bellman operator, or introducing explicit planning mechanisms. We evaluate Mesh-RL on hazard-dense grid-world environments with varying geometries and mesh resolutions. Across Q-learning, SARSA, and Dyna-Q, Mesh-RL consistently improves convergence speed, cumulative reward, and learning stability. Higher mesh resolutions sustain exploration, prevent premature convergence, and substantially accelerate value propagation to distant states. While Dyna-Q already benefits from internal planning, it still achieves additional gains under structured decomposition. Overall, Mesh-RL introduces a principled spatial domain-decomposition mechanism for accelerating temporal-difference learning. Our framework bridges finite element method-inspired boundary-consistency techniques from scientific computing with reinforcement learning to improve sample efficiency in sparse-reward environments. We will release source code of the study.
\end{abstract}

\section{Introduction}\label{int}
Temporal-difference (TD) reinforcement learning (RL) remains challenged by slow reward propagation and poor scalability in large or sparse-reward state spaces \citep{tsitsiklis1996analysis, sutton1998reinforcement, schnell2025temporal}. When rewards are delayed or spatially distant, standard TD updates propagate value information only locally, often requiring prohibitively many episodes to achieve accurate long-range credit assignment. Classical acceleration strategies such as prioritized sweeping \citep{moore1993prioritized} improve propagation efficiency by replaying important transitions, yet they remain fundamentally limited by the underlying unstructured state-space topology.

Hierarchical reinforcement learning (HRL) addresses slow propagation through temporal abstraction \citep{parr1997reinforcement, barto2003recent, pateria2021hierarchical, klissarov2025discovering}. Frameworks such as MAXQ \citep{dietterich2000hierarchical} and option-based architectures \citep{barto2003recent, sutton1999between, veeriah2021discovery, arcudi2025multi} decompose policies into temporally extended subroutines that shorten effective planning horizons. While highly effective for temporal credit assignment, these methods do not impose explicit spatial structure on the state space, nor do they provide mechanisms for enforcing consistent value propagation across spatial partitions.

Recent work has also explored alternative mechanisms for accelerating long-range credit assignment in sparse-reward settings. Value Iteration Networks \citep{tamar2016value} embed differentiable planning modules to explicitly propagate reward information across spatial state representations. Reward redistribution methods such as RUDDER \citep{arjona2019rudder} reshape delayed rewards into earlier informative signals, substantially

\begin{figure}[t]
    \centering
    \includegraphics[width=\linewidth]{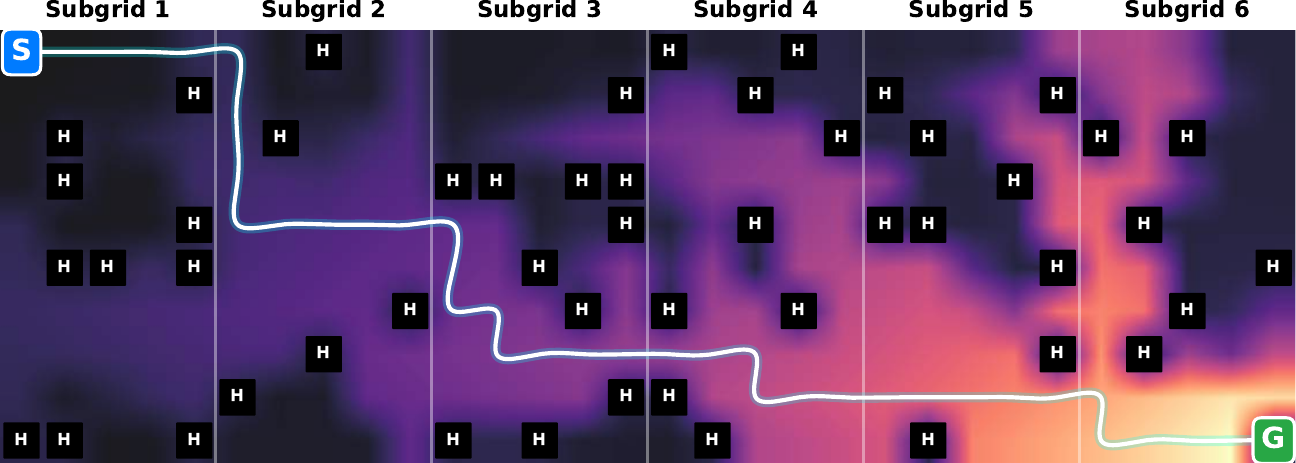}
    \caption{A graphical illustration of Mesh-RL.}
    \label{fig:fig111}
\end{figure}
improving TD learning in long-horizon tasks. Successor feature representations \citep{barreto2017successor} factorize value functions into predictive state occupancies, enabling faster generalization of reward signals across related tasks. Option-critic architectures \citep{bacon2017option} learn temporal abstractions end-to-end, reducing effective planning horizons and improving stability in sparse-reward domains. Exploration-driven planning frameworks such as Go-Explore \citep{ecoffet2019go} demonstrate that structured state-space traversal dramatically improves exploration in hard sparse-reward environments.

While these approaches enhance temporal or representational aspects of credit assignment, they do not explicitly impose spatial domain decomposition with boundary-consistent value coupling. Mesh-RL complements these directions by introducing structured spatial partitioning and deterministic boundary propagation into standard TD learning. In parallel, domain decomposition techniques in numerical analysis and scientific computing have long demonstrated that splitting large spatial domains into overlapping subdomains with boundary consistency constraints yields dramatic acceleration in solving partial differential equations \citep{toselli2004domain, taylor2013finite}. Inspired by this observation, we propose Mesh-RL, an RL framework that introduces structured spatial decomposition into TD learning. Mesh-RL partitions the environment into overlapping subgrids, introduces mandatory-passage boundaries, and performs boundary-aware TD updates that explicitly propagate value information across partitions. Such a learning strategy produces faster and more stable reward diffusion while remaining fully compatible with standard model-free TD algorithms.

Overall, the main contributions of our work can be summarized as follows:
\begin{itemize}[leftmargin=1.5em]
    \item We present Mesh-RL, a novel framework inspired by the finite element method that imposes boundary-consistent TD updates for coherent value propagation between neighboring subgrids, leading to efficient local computation while approximating a global solution;
    \item We show that Mesh-RL consistently improves convergence speed, cumulative reward, and overall learning stability across experiments on Q-learning, SARSA, and Dyna-Q.
\end{itemize}

Our paper is organized as follows. Section 2 covers related work following the Introduction. Section 3 describes how Mesh-RL fills the gap, primarily left by hierarchical RL and prioritized sweeping. The mathematical framework backing Mesh-RL is presented in Section 4. Section 5 outlines our methodology and experimental framework, and underlying empirical results, along with discussion and limitations are covered in Sections 6 and 7, respectively. Finally, Section 8 concludes the paper with the overall findings of the study.

\section{Related Work}

\textbf{Hierarchical and Decomposition-based RL.}
Hierarchical RL frameworks such as MAXQ \citep{dietterich2000hierarchical} and option-based architectures \citep{sutton1999between, barto2003recent} construct temporal abstractions that accelerate learning by decomposing long-horizon problems. Subsequent work on option discovery, nested abstraction, and multi-level HRL further improves temporal structure learning \citep{veeriah2021discovery, arcudi2025multi}. However, these approaches primarily introduce temporal rather than spatial abstraction and do not enforce structured value consistency across spatial regions.

Divide-and-conquer and state-space decomposition methods explicitly partition state spaces to train specialized policies or subgoals \citep{ghosh2017divide, sahni2017state, wong2021state}. Related graph-partition and bottleneck-discovery approaches identify critical states to guide exploration \citep{menache2002q, bacon2014bottleneck}. Yet these methods typically treat partitions as weakly coupled modules, merging policies or value functions after training rather than maintaining continuous boundary synchronization during learning.

\textbf{Reward Propagation Acceleration.}
Several methods aim to accelerate TD reward propagation without altering state-space structure. Prioritized sweeping \citep{moore1993prioritized} selectively replays transitions with high expected impact, while Dyna-style planning integrates model-based rollouts \citep{sutton1990first}. Proto-value functions and spectral methods construct global value bases reflecting state-space geometry \citep{mahadevan2005proto}. Intrinsic motivation and reward shaping accelerate exploration in sparse-reward settings \citep{ng1999policy, hare2019dealing}. These techniques improve efficiency but do not change the fundamental locality of TD propagation in spatially extended environments.

\textbf{Mesh-Structured and Scientific Computing-Inspired RL.}
Recent years have seen growing interest in applying RL to mesh-related scientific problems. Yang et al. \citep{yang2023reinforcement} formulate adaptive mesh refinement as an RL problem. Freymuth et al. \citep{freymuth2023swarm} propose Swarm RL for distributed adaptive mesh refinement using multi-agent coordination. Mojgani et al. \citep{mojgani2023multi} apply multi-agent RL to subgrid-scale closure modeling in fluid simulation. Wang et al. \citep{wang2023enhancing} exploit mesh-structured propagation for imagination-based RL, while Liu et al. \citep{liu2025mesh} introduce reinforcement fine-tuning for 3D mesh generation. These works demonstrate the usefulness of RL on meshes, but they do not use mesh-based domain decomposition to modify the internal mechanics of TD value propagation itself.

In summary, Mesh-RL combines three key elements derived from the finite element method and domain decomposition theory: (a) \textit{structured subgrid partitioning}, which enables modular and localized learning; (b) \textit{goal-directed update sequences}, which prioritize meshes closer to the goal for efficient reward propagation; and (c) \textit{boundary-aware Q-value synchronization}, which ensures global consistency while maintaining independent local learning. As illustrated in Figure~\ref{fig:fig111}, lighter colors represent the successful propagation of value estimates from the goal ($G$) back toward the start ($S$). These values are transported across subgrid boundaries through consistency constraints, allowing distant rewards to ``flow'' upstream more effectively than in unpartitioned spaces. Together, these features enable Mesh-RL to improve sample efficiency, accelerate convergence, and sustain exploration, particularly in sparse-reward environments where standard RL methods often struggle.

\section{How Mesh-RL Fills the Gap}
Mesh-RL introduces a spatially structured decomposition framework with explicit boundary-aware value coupling, filling a gap left by prior HRL and decomposition-based methods. Whereas hierarchical RL constructs temporal abstractions, Mesh-RL constructs spatial subgrids with mandatory-passage boundaries, enabling principled spatial credit assignment rather than purely temporal abstraction. Unlike divide-and-conquer and partition-based RL, which train submodules independently and merge them post hoc, Mesh-RL maintains continuous boundary synchronization through overlapping subgrid interfaces and boundary TD targets, yielding a tightly coupled global value function throughout learning.

Prioritized replay and model-based planning accelerate reward propagation by selecting important transitions \citep{moore1993prioritized}, but they do not modify the underlying spatial structure of propagation. Mesh-RL instead introduces a structural acceleration mechanism: boundary-aware TD updates that immediately transmit downstream value improvements upstream across spatial partitions. This produces faster reward diffusion without modifying the reward function or introducing auxiliary shaping signals.

Finally, while recent mesh-based RL work focuses on mesh refinement, mesh quality optimization, or simulation control \citep{yang2023reinforcement, freymuth2023swarm, mojgani2023multi, liu2025mesh}, Mesh-RL is the first framework to use mesh-inspired domain decomposition as an internal learning principle for RL itself. By treating subgrids as domain-decomposition elements and boundary TD constraints as discrete continuity conditions, Mesh-RL imports a proven scientific computing paradigm into TD learning. Importantly, it remains fully compatible with standard model-free algorithms such as Q-learning, SARSA as well the hybrid approach, Dyna-Q, preserving simplicity while providing a novel, physically inspired mechanism for scalable and spatial RL.

	\section{Mesh-RL: Mathematical Framework and Mesh-Based TD Learning}
	\label{sec:mesh-rl}
	Mesh-RL is inspired by finite element method (FEM), a cornerstone in computational physics and engineering, where complex physical domains are partitioned into smaller subdomains (elements) connected at nodes. In FEM \citep{taylor2013finite,hughes2003finite,bathe2006finite,reddy2005introduction}, governing partial differential equations (PDEs) are approximated locally using interpolation or shape functions, and boundary conditions enforce consistency across elements . This decomposition allows the solution of large, complex problems in a modular and efficient manner. Analogously, in Mesh-RL, the state space of a reinforcement learning environment is divided into $M$ subgrids, referred to as meshes or elements. Each mesh is treated as an independent learning unit where Bellman updates are performed locally. Updates begin in the mesh containing the goal and propagate outward, with boundary values synchronized through averaging to enforce global consistency. After each episode, the $Q$-values from all meshes are stitched together, and the averaged boundary values are redistributed to each mesh for the next learning iteration. This procedure ensures that local learning within each mesh is coupled, allowing structured and accelerated value propagation across the entire environment.
	
	\subsection{MDP Formulation}
	We model the environment as a Markov Decision Process 
	$(\mathcal{S}, \mathcal{A}, P, r, \gamma)$, where $\mathcal{S}$ is the state space (e.g., a 2D grid), 
	$\mathcal{A}$ the action space, $P(s'|s,a)$ the transition kernel, $r(s,a)$ the reward function, 
	and $\gamma \in (0,1)$ the discount factor.
	
	The optimal action-value function satisfies the Bellman optimality equation:
	\begin{equation}
		Q^*(s,a) = \mathbb{E}\Big[r(s,a) + \gamma \max_{a'} Q^*(s',a')\Big],
	\end{equation}
	and standard TD updates are:
	\begin{align}
		Q(s,a) &\leftarrow Q(s,a) + \alpha \delta(s,a),\\
		\delta(s,a) &= r(s,a) + \gamma \max_{a'} Q(s',a') - Q(s,a).
	\end{align}

	\subsection{Mesh Decomposition and Mandatory-Passage Property}
	
	The state space is partitioned column-wise into $M$ overlapping subgrids:
	\begin{equation}
		\mathcal{S} = \bigcup_{i=1}^{M} \mathcal{S}_i, \quad 
		\mathcal{S}_i = \{(r,c) \mid c \in [l_i,r_i)\}.
	\end{equation}
	Each subgrid $\mathcal{S}_i$ maintains a local action-value function $Q_i(s,a)$.
	Adjacent subgrids overlap on boundary regions 
	$\partial \mathcal{S}_i = \mathcal{S}_i \cap \mathcal{S}_{i+1}$.
	
	We consider grid environments where the start state lies in the leftmost subgrid 
	($s_{\text{start}}\in\mathcal{S}_1$) and the goal state lies in the rightmost 
	subgrid ($s_{\text{goal}}\in\mathcal{S}_M$). Since admissible transitions move 
	at most one column per step, any trajectory from $s_{\text{start}}$ to 
	$s_{\text{goal}}$ must cross each intermediate column. Consequently, every 
	successful trajectory intersects each boundary region $\partial \mathcal{S}_i$, 
	establishing a mandatory-passage property across subgrids.

	\subsection{Local and Boundary-Aware TD Updates}

	For transitions contained entirely within subgrid $\mathcal{S}_i$, Mesh-RL performs standard TD updates:
	\begin{equation}
		\begin{split}
			Q_i(s,a) &\leftarrow Q_i(s,a) + \alpha \Big[ \, r(s,a) + \gamma \max_{a'} Q_i(s',a') - Q_i(s,a) \, \Big]
		\end{split}
	\end{equation}
	
	For transitions reaching a boundary state $s_\partial \in \partial \mathcal{S}_i$, 
	downstream subgrid values are incorporated in the TD target:
	\begin{equation}
		\begin{split}
			Q_i(s,a) &\leftarrow Q_i(s,a) + \alpha \Big[ \, r(s,a) + \gamma \max_{a'} Q_{i+1}(s_\partial,a') - Q_i(s,a) \, \Big]
		\end{split}
	\end{equation}
	
	Defining 
	\begin{equation}
		V_i(s) = \max_a Q_i(s,a),
	\end{equation}
	the mandatory-passage property ensures
	\begin{equation}
		V^*(s) = \mathbb{E}\Big[r(s,a) + \gamma V^*(s_\partial)\Big], \quad s \in \mathcal{S}_i,
	\end{equation}
	so accurate downstream boundary estimates propagate value information backward to upstream subgrids.
	
	\subsection{Escape from Suboptimal Plateaus: Boundary Propagation Lemma}
	If a subgrid temporarily converges to a suboptimal local fixed point $Q_i^{sub}$ 
	satisfying negligible local TD error:
	\begin{equation}
		r(s,a) + \gamma \max_{a'} Q_i^{sub}(s',a') - Q_i^{sub}(s,a) \approx 0,
	\end{equation}
	but the downstream boundary improves $V_{i+1}(s_\partial)$, 
	the resulting boundary TD error becomes
	\begin{equation}
		\delta_i^{boundary}(s,a) = r(s,a) + \gamma V_{i+1}(s_\partial) - Q_i^{sub}(s,a),
	\end{equation}
	reintroducing positive updates and enabling escape from local plateaus.
	
	\begin{lemma}[Boundary Propagation]
		If $V_{i+1}^{(t+1)}(s_\partial) \ge V_{i+1}^{(t)}(s_\partial)$, 
		boundary TD updates in $\mathcal{S}_i$ produce nonnegative corrections to $Q_i$.
	\end{lemma}
	\begin{proof}[Proof sketch]
		TD targets depend linearly on $V_{i+1}(s_\partial)$; 
		thus, increasing downstream boundary values increases upstream TD targets, 
		yielding nonnegative updates.
	\end{proof}
	
	\subsection{Global Stitched Value Function}
	After local updates, the global value function is obtained by averaging overlapping subgrid estimates:
	\begin{equation}
		Q_G(s,a) = \frac{1}{|\{i : s \in \mathcal{S}_i\}|} 
		\sum_{i: s \in \mathcal{S}_i} Q_i(s,a),
	\end{equation}
	which is then synchronized back to subgrids before the next learning episode, enforcing global boundary consistency.
	
	\subsection{Bellman Fixed-Point Preservation and Convergence}
	Let $T$ denote the standard Bellman operator:
	\begin{equation}
		(TQ)(s,a) = r(s,a) + \gamma \mathbb{E}_{s'} \Big[\max_{a'} Q(s',a')\Big].
	\end{equation}
	
	\textbf{Proposition 1 (Fixed-Point Preservation).}
	Any fixed point of the Mesh-RL stitching update satisfies the Bellman optimality equation
	\[
	Q_G = T Q_G .
	\]
	
	\emph{Proof.}
	Each subgrid $\mathcal{S}_i$ maintains a local Q-function updated by its Bellman operator
	\[
	(T_i Q_i)(s,a)=r(s,a)+\gamma \mathbb{E}_{s'}[\max_{a'}Q_i(s',a')].
	\]
	Assume local convergence so that $Q_i = T_i Q_i$ for all $i$.  
	The global stitched Q is defined on overlapping states by
	\[
	Q_G(s,a)=\frac{1}{|\mathcal{I}(s)|}\sum_{i\in\mathcal{I}(s)}Q_i(s,a),
	\quad \mathcal{I}(s)=\{i: s\in\mathcal{S}_i\}.
	\]
	
	On overlapping states, all local Q-functions converge to the same optimal value:
	\[
	Q_i(s,a) = Q^*(s,a), \quad \forall i \in \mathcal{I}(s).
	\]
    
	Averaging identical values leaves them unchanged, so
	\[
	Q_G(s,a) = Q^*(s,a).
	\]

    \begin{figure}[t]
    \centering
    \includegraphics[width=0.90\linewidth]{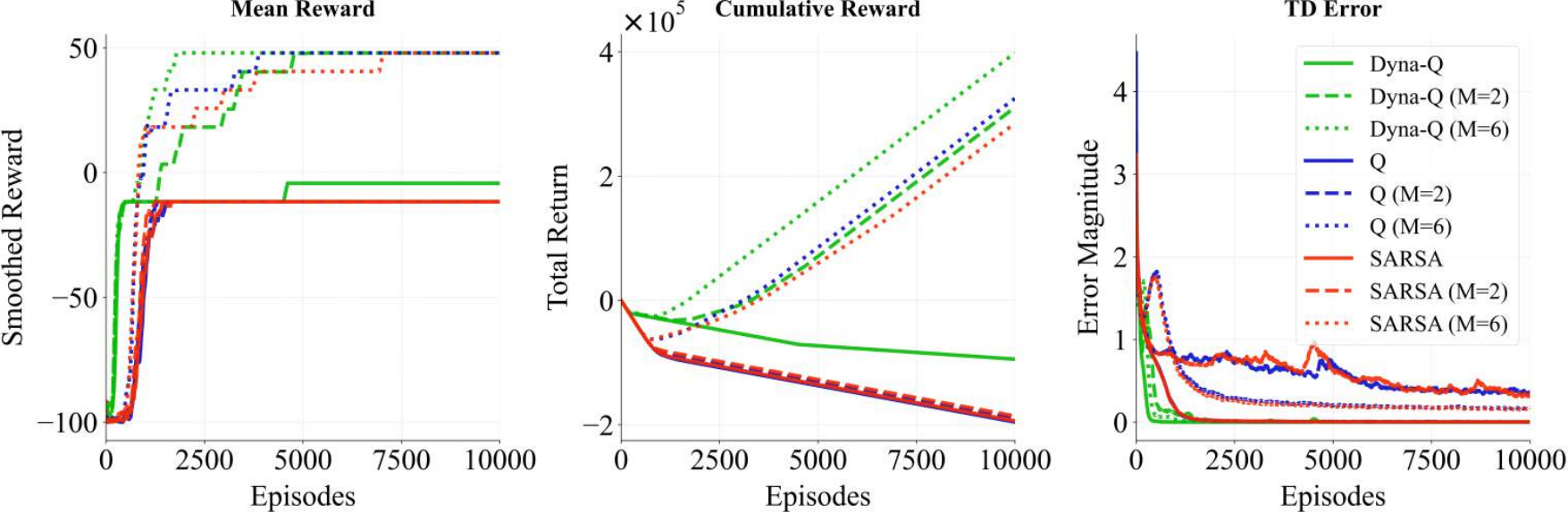}
    \caption{Performance on a $10\times30$ grid with 50 random holes ($M=2$ and $M=6$) and baseline methods.}
    \label{fig:fig2}
    \end{figure}
	
	Since \(Q^*\) satisfies the Bellman optimality equation \(Q^* = T Q^*\), it follows that
	\[
	Q_G = T Q_G.
	\]
	\hfill $\square$

\textbf{Proposition 2 (Contraction Property).}  
	The Mesh-RL operator $T_M$, which consists of local Bellman updates, boundary substitution, 
	and stitching on overlapping states, is a $\gamma$-contraction under the max norm:
	\[
	\|T_M Q - T_M Q'\|_\infty \le \gamma \|Q - Q'\|_\infty .
	\]
	
	\emph{Sketch of proof.}  
	Each local Bellman update is a $\gamma$-contraction.  
	The stitching step averages values on overlapping states, which cannot increase the maximum difference between two Q-functions.  
	Therefore, the overall Mesh-RL operator $T_M$ is also a $\gamma$-contraction.
	\hfill $\square$
	
	\textbf{Corollary.}  
	Under standard stochastic approximation conditions (finite MDP, Robbins–Monro learning rates, sufficient exploration), 
	Mesh-RL converges almost surely to the optimal action-value function $Q^*$, because $T_M$ shares the same fixed point as $T$ and is a contraction.
	
	\subsection{Relation to Classical Bellman Iteration}
	Standard Q-learning performs
	\begin{equation}
		Q^{(k+1)} = T Q^{(k)},
	\end{equation}
	while Mesh-RL performs
	\begin{equation}
		Q^{(k+1)} = T_M Q^{(k)},
	\end{equation}
	accelerating spatial propagation through domain decomposition and boundary synchronization, 
	analogous to Schwarz methods in PDE solvers. Thorough details of stitched TD learning are presented in Algorithm 1.

\section{Experiments}
The empirical evaluation of Mesh-RL examines the efficacy of this mesh-based decomposition when applied to classical RL algorithms. The approach draws parallels with domain decomposition methods in numerical PDE solvers, where complex problems are divided into manageable subdomains solved independently and then coupled via boundary conditions \citep{ toselli2004domain, gropp1995parallel}. Such methods are known to improve convergence speed, maintain solution consistency, and allow parallel computation, all of which inspire the design of Mesh-RL.

For baseline comparison, we evaluate Mesh-RL in conjunction with three fundamental RL algorithms. Q-Learning is an off-policy TD method that iteratively estimates the optimal action-value function via Bellman updates \citep{watkins1989learning, watkins1992q}. SARSA is an on-policy TD control method that updates action-values based on the action actually taken, making it more conservative in exploration \citep{ sutton1998reinforcement, rummery1994line}. Dyna-Q combines direct RL with planning through simulated experience, using $N$ planning steps per real interaction \citep{sutton1990integrated}. Dyna-Q’s hybrid nature allows it to propagate rewards efficiently even with fewer subgrid divisions, whereas Q-Learning and SARSA benefit more from structured mesh updates.

\begin{figure}[t]
    \centering
    \includegraphics[width=0.6\linewidth]{figs/fig2.png}
    \caption{Value heatmaps on a $10\times30$ grid with 50 holes, for baselines and Mesh-RL ($M=2$ and $M=6$).}
    \label{fig:fig3}
\end{figure}

The mesh decomposition strategy follows a systematic spatial partitioning approach: for $M$ meshes, the environment is divided into $M$ contiguous subregions, each containing approximately $N/M$ states, where $N$ is the total number of states. This partitioning aligns with principles from parallel computing and domain decomposition, where proper subdomain design is crucial for convergence and efficient information propagation \citep{toselli2004domain}.

Within each mesh, local Bellman updates proceed independently, while boundary condition enforcement ensures global consistency. Q-values at mesh interfaces are synchronized via averaging, analogous to Schwarz-type domain decomposition methods \citep{gropp1995parallel,gander2008schwarz}. Updates follow a goal-directed sequence: meshes containing the goal state are updated first, followed by upstream propagation toward the start state \citep{boyan1994generalization}. This ensures that early-stage learning efficiently propagates reward information throughout the environment, enhancing exploration and accelerating convergence.

All experiments employ consistent hyperparameters to maintain fairness: learning rate $\alpha=0.1$, discount factor $\gamma=0.99$, $\epsilon$-greedy exploration with $\epsilon=0.1$, step penalty $-1$, hole penalty $-10$, goal reward $+100$, episode budget per mesh = $100$, 10 independent seeds with different random initializations, and a 100-episode moving average for visualization. The sparse reward landscape, characterized by a dominant goal reward and substantial penalties for missteps, emphasizes the need for efficient exploration and credit assignment. This design ensures that Mesh-RL’s structured decomposition and boundary-aware updates can be rigorously evaluated for their ability to accelerate learning, improve reward propagation, and maintain stability across multiple reinforcement learning methods. All experiments have been performed on an Intel Core i9-11950H (2.60 GHz) CPU.

\section{Results and Analysis}

\subsection{10×30 Grid with 50 Holes}

The first experimental setup considers a 10×30 grid with 50 randomly placed holes, simulating a sparse-reward environment. We evaluated Mesh-RL with two mesh resolutions, $M=2$ and $M=6$, across Q-Learning, SARSA, and Dyna-Q. Figure~\ref{fig:fig2} shows the mean reward per episode averaged over 10 seeds. For $M=6$, Dyna-Q achieves the highest rewards most rapidly, consistent with its built-in planning efficiency. However, SARSA and Q-Learning also exhibit substantial improvements with $M=6$ compared to $M=2$, indicating that increased mesh resolution enhances both the speed and quality of learning. The larger number of meshes enables more localized updates, while boundary

\begin{figure}[t]
    \centering
    \includegraphics[width=0.90\linewidth]{figs/fig3.png}
    \caption{Performance on a $20\times20$ grid with 50 random holes ($M=2$ and $M=6$) and baseline methods.}
    \label{fig:fig4}
\end{figure}
synchronization ensures coherent propagation of reward information across the grid, effectively reducing the learning horizon to distant states. Cumulative reward plots reinforce this observation: higher mesh resolution leads to larger total reward accumulation. TD-error dynamics reveal that Mesh-RL maintains higher TD errors during early learning stages, reflecting sustained exploration and ongoing refinement of Q-values across subgrids. This contrasts with baseline methods, which often show early TD-error collapse and converge prematurely to suboptimal policies. Maintaining an elevated TD-error allows the agent to continue exploring promising trajectories while efficiently propagating information through the structured mesh.

Figure~\ref{fig:fig3} presents Q-value heatmaps for this\label{key} environment. For $M=6$, Mesh-RL clearly demonstrates smoother and more extensive propagation of value information across the grid. Mesh boundaries are visible in Q-Learning and SARSA, reflecting localized updates within each subgrid, yet the stitched boundary values produce globally consistent Q-values. Dyna-Q shows minimal differences between $M=2$ and $M=6$, as its planning steps already propagate reward information efficiently. These heatmaps highlight that structured decomposition benefits algorithms that rely on iterative TD updates, enabling faster learning and more uniform exploration.
\subsection{20×20 Grid with 50 Holes}

To further test Mesh-RL in a more balanced and moderately dense environment, we used a 20×20 grid with 50 random holes. Figure~\ref{fig:fig4} shows the mean reward per episode for $M=2$ and $M=6$. Across all methods, $M=6$ consistently accelerates convergence, allowing Q-Learning and SARSA to reach maximum rewards by episode 2000. Dyna-Q remains largely unaffected, as its planning mechanism efficiently propagates reward information even with fewer meshes. TD-error curves indicate that higher mesh resolution sustains elevated errors longer, confirming that subgrid decomposition encourages continued exploration while preventing premature convergence to suboptimal policies.

Figure~\ref{fig:fig5} presents Q-value heatmaps. For Q-Learning and SARSA, $M=6$ clearly enhances value propagation, producing smooth and globally consistent Q-fields. Even in early episodes, subgrid boundaries guide the propagation of value information from the goal to distant regions. This structured information transfer accelerates learning and allows the agent to anticipate rewards along optimal paths more efficiently. Dyna-Q exhibits minimal sensitivity to mesh resolution, highlighting that planning algorithms are inherently less dependent on structured decomposition.

Together, the reward curves, cumulative reward plots, TD-error dynamics, and heatmaps demonstrate that increasing mesh resolution systematically improves learning performance for TD-based methods. Higher mesh resolutions enable faster reward accumulation, more stable exploration, and effective propagation of Q-values across the environment, while Dyna-Q’s performance is largely unaffected due to its planning capability. These findings underscore the utility of Mesh-RL for algorithms that rely on incremental value updates, particularly in sparse-reward and challenging exploration settings.

\begin{figure}[t]
    \centering
    \includegraphics[width=0.8\linewidth]{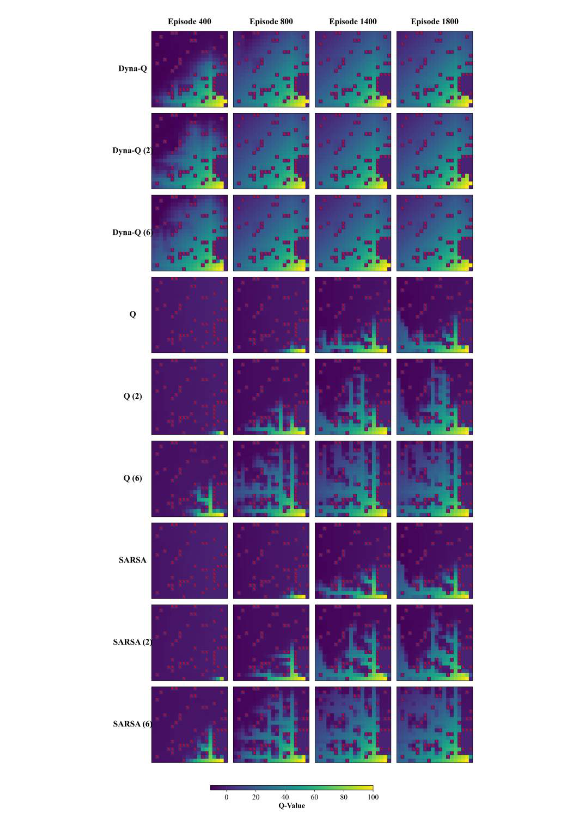}
    \caption{Value heatmaps on a $20\times20$ grid with 50 holes, for baselines and Mesh-RL ($M=2$ and $M=6$).}
    \label{fig:fig5}
\end{figure}

\subsection{Summary of Mesh Resolution Effects}

Across all experimental setups, increasing the mesh resolution from $M=2$ to $M=6$ consistently improves performance for TD-based methods. The benefits include:
\begin{itemize}
	\item \textbf{Faster reward accumulation:} Higher mesh resolution accelerates convergence by propagating values more efficiently across the environment.
	\item \textbf{Sustained TD-error:} Maintaining elevated TD-error early in learning ensures continued exploration and prevents premature convergence.
	\item \textbf{Structured Q-value propagation:} Heatmaps demonstrate that subgrid boundaries guide learning while stitched boundary values ensure global consistency.
	\item \textbf{Algorithm-dependent effects:} Mesh-RL provides the largest benefits for Q-Learning and SARSA, whereas Dyna-Q is relatively insensitive due to planning.
\end{itemize}
\begin{table}[h]
    \centering
    \caption{Aggregated Performance Summary (10 Seeds, 10,000 Episods): $20\times20$ vs. $10\times30$ Grid Environments (50 Holes)}
    \label{tab:tbl}
    \resizebox{\textwidth}{!}{%
        \begin{tabular}{@{}ll rrrrrrrr@{}}
            \toprule
            \textbf{Method} & \textbf{Grid} & \textbf{Accumulated\_Reward} & \textbf{Avg\_Reward} & \textbf{Final\_Reward} & \textbf{Peak\_Reward} & \textbf{Peak\_Ep} & \textbf{Peak\_Time} & \textbf{Total\_Time} & \textbf{Total\_Actions} \\ \midrule
            \multirow{2}{*}{Dyna-Q} & 20x20 & 531,434.3 & 53.14 & 63.0 & 63.0 & 620 & 23.99 & 457.42 & 271,939 \\
            & 10x30 & -94,400.0 & -9.44 & -4.3 & -4.3 & 620 & 13.56 & 247.79 & 45,547 \\ \cmidrule(l){2-10} 
            \multirow{2}{*}{Dyna-Q ($M=2$)} & 20x20 & 551,401.8 & 55.14 & 63.0 & 63.0 & 580 & 35.10 & 613.19 & 466,391 \\
            & 10x30 & 310,500.0 & 31.05 & 47.9 & 47.9 & 2287 & 72.48 & 343.60 & 362,937 \\ \cmidrule(l){2-10} 
            \multirow{2}{*}{\textbf{Dyna-Q ($M=6$)}} & 20x20 & \textbf{566,416.5} & \textbf{56.64} & \textbf{63.0} & \textbf{63.0} & \textbf{383} & 46.20 & 2142.99 & 1,227,903 \\
            & 10x30 & \textbf{398,200.0} & \textbf{39.82} & \textbf{47.9} & \textbf{47.9} & \textbf{906} & 44.18 & 518.08 & 939,048 \\ \midrule
            \multirow{2}{*}{Q} & 20x20 & -175,022.3 & -17.50 & 9.3 & 16.3 & 2019 & 72.18 & 342.29 & 213,367 \\
            & 10x30 & -195,500.0 & -19.55 & -11.7 & -11.7 & 511 & 11.92 & 219.94 & 53,679 \\ \cmidrule(l){2-10} 
            \multirow{2}{*}{Q ($M=2$)} & 20x20 & 105,416.1 & 10.54 & 47.5 & 55.0 & 2592 & 93.35 & 339.63 & 467,413 \\
            & 10x30 & -189,100.0 & -18.91 & -11.7 & -11.7 & 515 & 14.27 & 238.70 & 281,450 \\ \cmidrule(l){2-10} 
            \multirow{2}{*}{Q ($M=6$)} & 20x20 & 454,641.1 & 45.46 & 63.0 & 63.0 & 963 & 42.72 & 427.09 & 1,300,731 \\
            & 10x30 & 324,900.0 & 32.49 & 47.9 & 47.9 & 1354 & 33.16 & 242.43 & 1,055,048 \\ \midrule
            \multirow{2}{*}{SARSA} & 20x20 & -245,010.3 & -24.50 & 9.1 & 16.9 & 3227 & 108.57 & 322.13 & 200,015 \\
            & 10x30 & -193,600.0 & -19.36 & -11.7 & -11.7 & 489 & 11.95 & 225.28 & 53,687 \\ \cmidrule(l){2-10} 
            \multirow{2}{*}{SARSA ($M=2$)} & 20x20 & -10,918.9 & -1.09 & 39.8 & 47.1 & 2216 & 73.29 & 342.08 & 440,665 \\
            & 10x30 & -185,700.0 & -18.57 & -11.7 & -11.7 & 442 & 11.20 & 240.24 & 327,263 \\ \cmidrule(l){2-10} 
            \multirow{2}{*}{SARSA ($M=6$)} & 20x20 & 458,163.2 & 45.82 & 63.0 & 63.0 & 944 & 40.82 & 424.21 & 1,291,357 \\
            & 10x30 & 285,000.0 & 28.50 & 47.9 & 47.9 & 1805 & 39.76 & 234.55 & 1,035,946 \\ \bottomrule
        \end{tabular}%
    }
\end{table}

These observations confirm that the combination of subgrid partitioning, goal-directed updates, and boundary synchronization in Mesh-RL provides a systematic and principled mechanism for accelerating RL in sparse-reward and challenging environments.
\subsection{Final Reward}
As shown in Table~\ref{tab:tbl}, we first examine the final reward achieved after 10{,}000 training episodes. On the 20$\times$20 grid, the maximum attainable final reward is 63. For Dyna-Q-based methods, this optimal value is already achieved by the baseline planner, and introducing two or six meshes does not further improve asymptotic performance, indicating that classical Dyna-Q planning already provides strong reward propagation. In contrast, Q-learning and SARSA show substantial gains as the number of meshes increases. With six meshes, both methods reach the optimal final reward of 63, whereas their baseline counterparts remain far below this level. The same trend holds on the 10$\times$30 grid, demonstrating that MeshRL significantly enhances value propagation and credit assignment for model-free methods.

\subsection{Peak Reward and Convergence Speed}
We next analyze peak reward and the episode at which it is reached. For the 20$\times$20 grid, Dyna-Q with six meshes attains its peak reward after only 383 episodes, and for the 10$\times$30 grid after 906 episodes, indicating rapid convergence. For Q-learning and SARSA, increasing the mesh count both raises the achievable peak reward and substantially reduces the number of episodes required to reach it. This shows that MeshRL not only improves ultimate performance but also accelerates the learning dynamics, enabling model-free agents to approach optimal policies significantly faster.

\subsection{Computational Efficiency}
Finally, we evaluate peak time, total runtime, and total action counts. While MeshRL increases the number of executed actions due to subgrid exploration (e.g., over 1.2 million actions for six-mesh variants versus roughly 0.2--0.3 million for baselines), it often reduces wall-clock time to reach peak performance. For example, on the 20$\times$20 grid, Q-learning with six meshes reaches the optimal reward in 42.7 seconds, whereas baseline Q-learning fails to reach this reward even after 342 seconds of training. Similar trends are observed for SARSA. For Dyna-Q, peak time moderately increases with mesh count due to additional subgrid coupling, but without affecting final performance. Overall, Mesh-RL achieves faster convergence and higher rewards, particularly for model-free methods, while maintaining practical computational cost.

\begin{algorithm}[h]
	\caption{Stitched TD Learning with Subgrid Decomposition}
	\label{alg:mesh_rl}
	\begin{algorithmic}[1]
		\STATE \textbf{Initialize:} Grid $G(R, C)$, Episodes $E$, Base Max Steps $L$, Subgrids $n$, Learning rate $\alpha$, Discount $\gamma$
		\STATE Initialize $Q_{stitch}$ and $Q_{1 \dots n}$ to zeros
		
		\FOR{episode $e = 1$ to $E$}
		\STATE // \textit{Stage 1: Local Subgrid Learning (Right to Left)}
		\FOR{subgrid $i = n$ down to $1$}
		\STATE Determine local $s_{start}$ and $s_{goal}$ (boundary-dependent)
		\STATE $local\_limit \gets \lfloor L / n \rfloor$
		\STATE $s \gets s_{start}$
		
		\FOR{step $t = 1$ to $local\_limit$}
		\STATE $a \gets \epsilon\text{-greedy}(Q_i, s)$
		\STATE $s', r, done \gets \text{EnvironmentStep}(s, a)$
		
		\IF{not $done$}
		\STATE $Target \gets r + \gamma \max_{a'} Q_i(s', a')$
		\ELSE
		\STATE $Target \gets r$
		\ENDIF
		
		\STATE $Q_i(s, a) \gets Q_i(s, a) + \alpha (Target - Q_i(s, a))$
		
		\IF{method is DYNA}
		\STATE Perform $PlanningSteps(Q_i, Model)$
		\ENDIF
		
		\STATE $s \gets s'$
		\IF{$s = s_{goal}$ or $done$} 
		\STATE \textbf{break} 
		\ENDIF
		\ENDFOR
		
		\IF{$i > 1$}
		\STATE Copy shared boundary $Q_i$ to $Q_{i-1}$
		\ENDIF
		\ENDFOR
		
		\STATE // \textit{Stage 2: Stitching and Synchronization}
		\STATE $Q_{stitch} \gets \text{Average}(Q_{1 \dots n})$ for all overlapping states
		\FOR{subgrid $i = 1$ to $n$}
		\STATE $Q_i \gets Q_{stitch}$ 
		\ENDFOR
		
		\STATE // \textit{Stage 3: Global Evaluation}
		\STATE Rollout greedy policy from global $START$ using $Q_{stitch}$ for $L$ steps
		\ENDFOR
	\end{algorithmic}
\end{algorithm}

 \section{Discussion and Limitations}

Mesh-RL is motivated by the observation that TD learning propagates value information locally, leading to slow reward diffusion in large or sparse-reward environments. By decomposing the state space into overlapping subgrids and
enforcing boundary-consistent TD updates, Mesh-RL introduces structured long-range communication between distant regions while preserving the simplicity of standard TD learning. Our empirical results show that finer mesh resolutions systematically accelerate convergence and improve stability, confirming that controlled boundary synchronization effectively mitigates delayed credit assignment.

The interaction between Mesh-RL and planning-based methods such as Dyna-Q provides additional insight. While Dyna-Q benefits from simulated experience, Mesh-RL consistently yields further gains, indicating that spatial decomposition and model-based planning address complementary aspects of value propagation. Visualizations of value functions and TD-error dynamics suggest that Mesh-RL reduces stagnation in distant regions of the state space and prevents premature convergence to suboptimal policies.

Despite these promising results, being a preliminary study, Mesh-RL focuses on tabular TD methods in discrete grid-world environments, where this aspect is a limitation of our work. Mesh construction parameters, including subgrid size and overlap width, are manually specified, and the current formulation does not yet address function approximation, continuous control, or highly stochastic dynamics. Hence, as future work, we will extend Mesh-RL to deep RL with neural value approximation, investigate adaptive or learned mesh partitioning strategies, and develop theoretical convergence analyses under general approximation settings. Complementing latter, applying Mesh-RL to large-scale navigation and robotics domains is a promising direction in this regard, which we also leave for future exploration.

\section{Conclusion}

In this paper, we have presented \emph{Mesh-RL}, an RL framework that accelerates TD learning through spatial domain decomposition and boundary-aware value propagation. By partitioning the state space into overlapping subgrids and enforcing boundary-consistent TD updates, Mesh-RL enables localized learning while maintaining globally coherent value functions, improving long-range credit assignment without modifying the reward function or Bellman operator. Experiments on $10\times30$ and $20\times20$ grid-worlds with sparse rewards demonstrate that increasing mesh resolution systematically improves convergence speed, cumulative reward, and learning stability for TD-based methods such as Q-learning and SARSA, while Dyna-Q remains strong due to its internal planning. Value heatmaps and TD-error dynamics further confirm that structured boundary synchronization accelerates reward diffusion across distant regions of the state space. Overall, Mesh-RL opens a new avenue for bridging classical domain-decomposition principles from scientific computing with RL, providing a simple and general mechanism for scalable spatial value propagation. 

\section*{Impact Statement}
This paper presents work aimed at advancing the field of machine learning by proposing improved optimization and normalization mechanisms for reinforcement learning. There are many potential societal consequences of advances in machine learning; however, we do not identify any immediate negative ethical or societal risks arising uniquely from this work beyond those commonly associated with general-purpose machine learning research.

\bibliographystyle{unsrt}
\bibliography{main}

\end{document}